\documentclass[11pt,a4paper]{article}
\usepackage[margin=1.2in]{geometry}
\usepackage{graphicx} 

\usepackage{notation}

\usepackage{graphicx} 

\usepackage[utf8]{inputenc} 
\usepackage[T1]{fontenc}
\usepackage{lmodern}

\usepackage{microtype}
\usepackage{graphicx}
\usepackage{subfigure}
\usepackage{booktabs} 
\usepackage{thm-restate}
\usepackage{cancel}
\usepackage{wrapfig}
\usepackage{algorithm}
\usepackage[noend]{algpseudocode}

\usepackage{amsmath}
\usepackage{amsthm}
\usepackage{amssymb}
\usepackage{mathtools}
\usepackage[capitalize,noabbrev]{cleveref}

\usepackage{amsthm}
\usepackage{cleveref}

\theoremstyle{plain}
\newtheorem{theorem}{Theorem}
\newtheorem{proposition}{Proposition}
\newtheorem{lemma}{Lemma}
\newtheorem{corollary}{Corollary}

\theoremstyle{definition}
\newtheorem{definition}{Definition}

\theoremstyle{remark}

\crefname{theorem}{theorem}{theorems}
\Crefname{theorem}{Theorem}{Theorems}
\crefname{proposition}{proposition}{propositions}
\Crefname{proposition}{Proposition}{Propositions}
\crefname{lemma}{lemma}{lemmas}
\Crefname{lemma}{Lemma}{Lemmas}
\crefname{corollary}{corollary}{corollaries}
\Crefname{corollary}{Corollary}{Corollaries}
\crefname{definition}{definition}{definitions}
\Crefname{definition}{Definition}{Definitions}
\crefname{assumption}{assumption}{assumptions}
\Crefname{assumption}{Assumption}{Assumptions}
\crefname{example}{example}{examples}
\Crefname{example}{Example}{Examples}
\crefname{remark}{remark}{remarks}
\Crefname{remark}{Remark}{Remarks}

\newcommand{\R}{\mathbb{R}}
\newcommand{\E}{\mathbb{E}}

\allowdisplaybreaks

\usepackage{xcolor}

\twocolumn

\title{Microeconomic Foundations of Multi-Agent Learning}
\author{Nassim Helou}
\date{Harrisburg University, Harrisburg, USA \\
Thatch, New York, USA \\
January 2026}

\usepackage{natbib}
\begin{document}

\maketitle

\begin{abstract}
Modern AI systems increasingly operate inside markets and institutions where data, behavior, and incentives are endogenous. This paper develops an economic foundation for multi-agent learning by studying a principal--agent interaction in a Markov decision process with strategic externalities, where both the principal and the agent learn over time. We propose a two-phase incentive mechanism that first estimates implementable transfers and then uses them to steer long-run dynamics; under mild regret-based rationality and exploration conditions, the mechanism achieves sublinear social-welfare regret and thus asymptotically optimal welfare. Simulations illustrate how even coarse incentives can correct inefficient learning under stateful externalities, highlighting the necessity of incentive-aware design for safe and welfare-aligned AI in markets and insurance.
\end{abstract}

\section{Introduction}

Artificial intelligence is no longer a technology acting in isolation, but an economic force embedded inside markets, institutions, and large--scale systems. Modern AI systems—large foundation models, multi-agent simulators, autonomous decision-makers, data markets, and algorithmic insurers -- operate in environments filled with strategic actors whose objectives shape the data and information flows on which AI relies. With the increasing deployment of such systems, one may wander how the interactions of such systems can be made oriented towards greater social welfare. Thus, a central challenge emerges: are there ideas and concepts from economic theory that could be employed in order to improve such systems and ? Typical and important questions are: how should AI reason about incentives, how should economic mechanisms shape the behavior of learning agents, are there insights from game theory that may explain some learning and decision--making behaviors? Understanding this interface is essential for ensuring that AI systems behave safely and align their target \emph{policy} with social welfare. All this will also help to understand if deployed algorithms are robust to strategic behaviors and collusion.
\newline
\newline
Insurance automated markets are particularly impacted because they sit at the intersection of prediction, incentives, and strategic behavior, all of which are fundamentally altered once learning algorithms become central decision-makers. Premiums, deductibles, coverage limits, and exclusions all act as incentives that influence reporting, risk-taking, and even underlying risk itself. Hence markets and insurance are very impacted by all the aspects mentioned above. However, the literature in \textit{Machine Learning} is not totally developed in this direction, hence our point in this paper. Traditional insurance is built on statistical estimation, risk pooling and contract design. But as AI becomes the key agent mediating risk predictions, dynamic pricing, and fraud detection, these tasks evolve into multi-agent learning problems with strategic participants: customers learn how to respond to pricing signals, large models learn risk distributions from data influenced by behavior, and insurers must design incentive structures that maintain truthful reporting or appropriate risk. 
Emerging AI--driven marketplaces (for instance automated markets, autonomous logistics networks, online advertising auctions) exhibit the same structural challenges: economic externalities, strategic information revelation, and misaligned exploration incentives. Markets in which AI systems interact with humans and with each other are, fundamentally, \emph{games of learning agents}.
\newline
\newline
This transformation exposes a profound theoretical gap. Classical machine learning assumes that data is exogenous and agent behavior is myopic. Classical economics assumes known environments and fully rational optimization. But in the settings above, neither assumption holds: agents learn, adapt, explore, and manipulate each other in an unknown environment. Leveraging the vocabulary from mechanism design, we can now call the platform (insurer, regulator, etc) the \emph{principal} and the other players in interaction the \emph{agents}. Since utility functions or the environment are unknown, the principal must learn and make decisions simultaneously. At the same time, the environment evolves as a consequence of these learning processes. The result is a new regime of uncertain and strategic environment \citep{rothschild2025agentic, immorlica2024generative} made of interacting learners. Tools from reinforcement learning \citep{kaelbling1996reinforcement, sutton1999reinforcement}, contract theory, game theory, and mechanism design must be fused at a core level to provide valuable insights.
\newline
\newline
Motivated by the scarcity of work at the intersection of industry actors (e.g., insurers, online platforms) and academia—which tends to focus either on classical game theory or on more pure ML-oriented research—we aim to formulate several core questions and provide initial algorithmic insights. This paper contributes to this agenda by developing a unified framework for studying incentive--compatible learning in multi--agent systems, grounded in the economic theory of contracts and externalities and in modern tools from online learning and statistics. We start from the observation that AI systems increasingly function as principals that must elicit information and effort from human or artificial agents whose internal objectives, learning dynamics, and types are unknown. As shown in the literature on contract design \citep{guesnerie1984complete, bolton2004contract, kHoszegi2014behavioral} and delegated learning \citep{saig2023delegated}, incentives \citep[see the very extensive book, ][]{laffont2002theory} shape statistical performance and exploration behavior in essential ways. For instance, when data collection is delegated to learning agents, the principal must account for both hidden states and hidden actions, designing transfer schemes that remain robust despite noisy evaluation \citep{ananthakrishnan2024delegating}. When agents face costly exploration, standard RL algorithms violate incentive compatibility, requiring information--design mechanisms to ensure proper exploration \citep{simchowitz2024exploration}. Likewise, externalities, moral hazard, and strategic manipulation appear in repeated bandit and MDP settings \citep{scheid2024learning}, emphasizing how classical economic forces reemerge in learning environments.
\newline
\newline
Our work builds on these insights and pushes them into a genuinely dynamic and stateful setting. As formalized in this paper, we consider a Markov decision process (MDP) \citep{bellman1957markovian, puterman1990markov} in which both the principal and the agents are learning over time, and where the agent’s actions influence not only their own returns but also the principal’s reward and the transition dynamics of the system. This environment captures essential features of AI in the context of data--powered markets and insurance systems: feedback loops between predictions and behavior, exploration that may impose externalities, and incomplete information about agent preferences. In the context of theoretical works in the field of statistics, feedback loops where a predictor influences the system from which it learns is increasingly studied as \emph{performative prediction}, as in \citet{perdomo2020performative, mendler2020stochastic} or \citet{brown2022performative}. In such a system, classical efficiency theorems break down unless the principal can infer the agent’s learning dynamics and design transfers that internalize externalities. We show that, despite these challenges, a carefully constructed two--phase mechanism yields asymptotically optimal social welfare: the principal can first learn how to influence the agent and then use this influence to steer long--run favorable outcomes.
\newline
\newline
Crucially, we extend these ideas beyond contract--based systems. Recent works reveals that the generative modeling techniques from diffusion models can be interpreted as economic aggregation mechanisms, implementing welfare --maximizing estimators and equilibrium decision rules. We formalize this connection and show that the denoising step of diffusion models corresponds to the unique solution of a social planner problem, and can be implemented as an equilibrium in a large--agent economy. This offers a surprising connection between economic theory and state of the art generative AI: diffusion models perform a form of efficient market aggregation. When integrated with principal--agent learning, this provides whole new ideas for designing collaborative AI systems in economic terms and mapping them to generative models.
\newline
\newline
Putting these elements together, this paper argues for a future in which AI systems behave as economic institutions—mediators of incentives, coordinators of decentralized learners. Questions then arise about how such systems can be oriented towards welfare optimization in environments shaped by strategic feedback. Nowhere is this more relevant than in insurance, where AI is used to evaluate risks, generate contracts or detect fraud. Consumer behavior could even be framed as a large multi--agent learning problem. In such environments, insurance cannot be merely actuarial; it must be algorithmic, incentive--aware, and grounded in learning dynamics. Our framework provides both theoretical foundations and practical insights from the industry toward this vision.
\newline
\newline
In summary, this work provides (i) a rigorous model of principal–agent learning in MDPs with endogenous externalities, (ii) welfare and regret guarantees for incentive-compatible mechanisms in dynamic systems, and (iii) a conceptual and mathematical bridge between diffusion models and economic aggregation. Together, these contributions shed light about how AI and economics must be tightly linked to build safe and and welfare--aligned systems for the markets and insurance infrastructures of the future.

\section{Related Work}

Before diving into the model, we review some important works linked with our setting. The study of learning and incentives in multi-agent systems lies at the intersection of contract theory and modern reinforcement-learning approaches. Classical principal–agent theory provides the foundation: beginning with the seminal formulations of hidden-action and hidden-information problems \citep{mirrlees1999theory}, the economic literature characterizes how a principal induces an agent to take costly, unobservable actions by offering outcome-contingent transfers. These models traditionally assume static or small dynamic environments with full knowledge of outcome distributions. Their central contribution is the articulation of incentive-compatibility constraints, participation constraints, and the structure of optimal contracts when types, costs, or actions are not directly observed. Our work inherits this conceptual logic but extends it to environments where both the principal and the agents are learning players, repeatedly interacting together in large state spaces under uncertainty about transition dynamics and reward structures.
\newline
\newline
A first major strand of work extends contract theory into algorithmic and high-dimensional domains. Combinatorial contracts \citep{dutting2022combinatorial} and multi-agent contract design \citep{dutting2023multi} study settings where the principal’s reward depends on combinatorial interactions between multiple agents or many possible effort dimensions. These works develop approximation schemes, impossibility results, and structural characterizations of optimal linear or bounded contracts in complex environments. More recent results show how contract classes can be understood through their pseudo-dimension, yielding sample-complexity guarantees for offline learning of near-optimal contracts from agent-type datasets \cite{dutting2025pseudo}. In the same direction, some works explore the trade-offs between expressiveness and learnability of menus and piecewise-linear contracts. Together, this literature establishes the algorithmic foundations of large-scale contract design.
\newline
\newline
A growing line of research integrates contract theory with online learning. Several works investigate repeated principal–agent interactions under bandit feedback, in which the principal observes only the stochastic realization of outcomes but not the agent’s type or reward function. Online contract-learning frameworks \citep{scheid2024incentivized, liu2025learning} study how a principal can ensure incentive compatibility while simultaneously learning unknown rewards. Similar ideas appear in more complex multi-agent structures in tree-like graphs: \citet{scheid2025online,scheid2025learning} demonstrate that local, one-step transfers suffice to globally steer all players toward the optimal joint action, effectively achieving welfare maximization in fully decentralized systems. These works show that without structural assumptions on agent response, principal regret is necessarily linear, while mild restrictions—such as empirically-greedy or elimination-based behavior—recover sublinear regret. More sophisticated models incorporate strategic learning on the agent side: \citet{liu2024principal} study agents who maintain their own empirical estimates and may explore arbitrarily, proving nearly optimal regret bounds for robust incentivization \citep{liu2024principal}. Complementing these works, \citet{wu2025learning} introduce a general \emph{learning to lead} model where the agent may strategically manipulate the principal’s learning by misreporting or inducing misleading observations. These results collectively highlight the delicate interplay between incentive compatibility and statistical learning, a theme central to our paper.
\newline
\newline
Recent work emphasizes delegation of learning tasks and the design of incentives affecting data quality or exploration incentives. A notable direction studies delegated data collection in decentralized or federated environments. \citet{ananthakrishnan2024delegating} show that when the principal relies on strategic agents to collect data that will later be used for training, both hidden actions and hidden states arise naturally, and performance-based contracts can achieve near-optimal delegation despite uncertainty in rewards and data quality \citep{ananthakrishnan2024delegating}. Closely related are incentive-compatibility constraints in exploration: in reinforcement-learning settings where exploration is costly for the agent, \citet{simchowitz2024exploration} demonstrate that standard RL algorithms violate classic incentive compatibility, and that exploration must be orchestrated using controlled information disclosure rather than monetary transfers \citep{simchowitz2024exploration} while \citet{capitaine2024unravelling} study how a principal can orchestrate data collection by agents in the purpose of collaborative learning when such collection is costly to the agents. Earlier principal–agent bandit models take the opposite view: the principal directly pays agents to explore, allowing the principal to learn unknown reward functions \cite{scheid2024incentivized}. Our work follows this line of thought but embeds the interaction inside a Markovian system and allows both sides to learn.
Another major direction concerns incentive problems arising from externalities and coordination failures. For the fixed and fully rational scenario, results from the Coase theorem have existed for decades \citep{coase2013problem, medema2020coase, farrell1987information, deryugina2021environmental}. Previous works have been developed to extend the setup to a game in an unknown environment with learning. In two-agent bandit settings, \citet{scheid2024learning} show that without property rights, welfare-maximizing outcomes may be impossible because agents fail to internalize externalities; surprisingly, appropriate transfer schemes restore an online analogue of the Coase theorem. \citet{zuo2024new} extends these ideas to dynamic environments with learning and uncertainty, emphasizing the importance of online bargaining and stability notions \citep{zuo2024new}. Fairness considerations have also entered the literature: \cite{tluczek2025fair} show that linear contracts can be adapted to satisfy fairness constraints across heterogeneous agents while preserving sublinear regret and high welfare in repeated interactions. Such works illustrate how classical concepts—externality internalization, bargaining, and fairness—must be reinterpreted when agents are learners rather than fully rational optimizers.
\newline
\newline
Separately, recent works connect principal–agent reasoning with reinforcement learning in MDPs and Markov games. \citet{ivanov2024principal} propose principal–agent reinforcement learning, introducing a meta-algorithm that converges to subgame-perfect Nash equilibrium (SPNE) in principal–agent MDPs through alternating optimization over policies. They show that contract-based payments can be interpreted as a form of reward shaping with principled economic meaning, and that deep RL can scale such mechanisms to large MDPs. Extensions to multi-agent Markov games \citep[see, e.g. ][ for a general overview]{littman1994markov, nowe2012game, zhang2021multi, yang2020overview} demonstrate how contract-based interventions can mitigate sequential social dilemmas in environments such as the Coin Game. Complementary, extensions of such setups to mean-field games have been studied \citep{lasry2007mean, bensoussan2013mean, guo2019learning}, where a mediator incentivizes a large population of no-regret agents towards desired equilibria despite model uncertainty \citep{widmer2025steering}. These results highlight the importance of learning-based incentive design in sequential and population-scale environments, foreshadowing the complexity of future AI ecosystems.
\newline
\newline
Finally, these lines of work are closely connected to the literature on experimental design \citep{kirk2009experimental, berger2018experimental, federer1956experimental}, which studies how data should be selected in order to maximize statistical efficiency and downstream decision quality. In classical statistics, experimental design formalizes the trade-off between information acquisition and resource constraints. In modern machine learning, these concerns reemerge in adaptive, sequential, and interactive settings, where data is shaped by the behavior of learning agents and by the incentives embedded in the system. This perspective links incentive design, exploration, and data collection to optimal design principles, which frame learning as an optimization problem over information structures rather than a single estimation task. Optimal design \citep{atkinson2014optimal, goos2016optimal} has appeared to be fundamental as a tool to select which data can be useful for training and inference. In the perspective on reward-model training in RLHF \citep{wang2024secrets, wang2024comprehensive, fu2025reward}, the selection of human-labeled preference pairs can be framed as a pure-exploration bandit problem \citep{zhao2024sharp, scheid2024optimal}. By characterizing simple regret and constructing matching upper and lower bounds, it can be shown how incentives (in this case, allocation of costly human annotation effort) shape the statistical efficiency of reward inference. Principal–agent learning problems can be interpreted as instances of endogenous experimental design, in which mechanisms and transfers determine not only agent behavior but also the statistical efficiency of learning itself—a theme that directly motivates and complements the framework developed in this paper.
\newline
\newline
Overall, these lines of research converge towards a central insight: \emph{as AI systems increasingly consist of interacting agents whose incentives and information are distributed, classical contract theory must be fused with online learning} to develop modern and fair systems. This paper contributes to this synthesis by studying principal–agent interactions in Markovian environments with learning on both sides, demonstrating how incentive design, exploration strategies, and multi-agent coordination interact in dynamic and uncertain settings.



\section{Incentive Design in a Principal Agent MDPs with Externalities}

As AI systems increasingly mediate economic activity, social coordination, and large-scale decision processes, understanding how learning agents interact strategically becomes essential for ensuring that these systems behave safely, efficiently, and fairly.
The theoretical framework developed in this work provides a principled foundation for addressing these challenges, showing how economic mechanisms can be integrated into learning systems so that individual agents contribute to globally desirable outcomes. By demonstrating that social welfare can provably be recovered—even when the AI system does not control the environment directly and must infer the preferences and learning dynamics of other agents— we hope that this research opens the door to designing AI platforms that can steer decentralized ecosystems without coercion or unrealistic assumptions about agent rationality.
\newline
\newline
Such results are critical as AI continues to move from laboratory settings into open, complex markets: ride-sharing platforms, generative-model marketplaces, multi-agent simulation environments, collaborative robots are structured around incentives rather than direct control. Understanding how to design transfers, bargaining schemes \citep{muthoo1999bargaining, powell2002bargaining, staahl1973bargaining}, and incentive-compatible protocols allows to predict and regulate how agents behave, reducing risks of exploitation or welfare collapse. Equally importantly, these insights support the development of AI that can reason about incentives, negotiate with humans, and commit to fair and transparent mechanisms that align behavior across diverse stakeholders. As online Coasean results generalize from simple bandits to rich MDP environments, we gain not only new theoretical guarantees but also a conceptual roadmap for building AI systems that combine learning, contracts, and strategic reasoning. This work highlights the necessity of merging economics and online learning at a fundamental level and helps ensure that the next generation of AI technologies can thrive within the multi-agent, incentive-driven world in which they will inevitably operate.
\newline
\newline
We extend the externality and bargaining framework developed in the bandit setting to a MDP in which the agent controls the environment while the principal influences the agent's behavior through transfers. The agent's actions determine both the principal's reward and the transition probabilities, and both players learn over time. As in the online bargaining linked with bandits, the principal seeks to internalize externalities through dynamic transfer policies, but the MDP structure introduces a longer exploration phase and more complex learning dynamics.


\subsection{Setting}

The environment is a finite-horizon MDP with a state space $\cS$, $|\cS| = S$, action space $\cA$, $|\cA| = K$. For any states $s,s' \in \cS$, action $a \in \cA$, we have a transition kernel $P(s' \mid s,a)$; agent reward $r_a(s,a) \in [0,1]$ and principal reward $r_p(s,a) \in [0,1]$.
\newline
\newline
Episodes have horizon $H$. At episode $k$ and step $h$, the state is $s_h^k$. First, the principal chooses a transfer vector $\tau_h^k(\cdot) \in \R_+^K$, then the agent takes action $a_h^k \in A$, the agent receives a reward:
\[r_a(s_h^k,a_h^k) + \tau_h^k(a_h^k) \eqsp,\]
while the principal's reward is
\[r_p(s_h^k,a_h^k) - \tau_h^k(a_h^k) \eqsp,\]
where the transfers add up to each of the players' utilities at each round. If one has in mind the setting of an insurer (the principal) and a client (the agent), the incentives would typically be discounts or promotions offered to the client for some advantageous contracts. Finally, the transition is $s_{h+1}^k \sim P(\cdot \mid s_h^k,a_h^k)$.
\newline
\newline
Being rational, the agent maximizes the expected cumulative return
\[
\E\Bigg[ \sum_{k=1}^T \sum_{h=1}^H \bigl(r_a(s_h^k,a_h^k) + \tau_h^k(a_h^k)\bigr) \Bigg] \eqsp,
\]
and the principal maximizes
\[
\mathbb{E}\Bigg[ \sum_{k=1}^T \sum_{h=1}^H \bigl(r_p(s_h^k,a_h^k) - \tau_h^k(a_h^k)\bigr) \Bigg] \eqsp.
\]
Summing the utilities obtained by the players, we define the social welfare as
\[
W = \E\Bigg[ \sum_{k=1}^T \sum_{h=1}^H \bigl(r_a(s_h^k,a_h^k) + r_p(s_h^k,a_h^k)\bigr) \Bigg] \eqsp,
\]
since transfers cancel there. Hence, the transfers are used to shape the players' behaviors but do not account in the global welfare. Their only use is to align the players' utilities.


\paragraph{Policies and Social Welfare.} A stationary agent policy is a mapping $\pi_a : \cS \to \Delta(\cA)$, and a stationary transfer policy is a mapping $\pi_\tau : \cS \to \R_+^K$.
\newline
\newline
Let $V_a^{\pi_a,\pi_\tau}$ and $V_p^{\pi_a,\pi_\tau}$ denote the value functions for the agent and principal. Social welfare under $(\pi_a,\pi_\tau)$ is
\[
W(\pi_a,\pi_\tau) = V_a^{\pi_a,\pi_\tau} + V_p^{\pi_a,\pi_\tau} \eqsp.
\]
We thus define the optimal global welfare as
\[
W^\star \coloneqq \max_{\pi_a} W(\pi_a,\pi_\tau) \eqsp,
\]
noting that transfers do not affect welfare.


\subsection{Players' Behaviors}

\paragraph{Agent Learning and Rationality.} Let $\pi_{a,k}$ be the agent's policy in episode $k$, produced by a reinforcement learning algorithm that updates from past episodes. We impose an episodic regret assumption analogous to a form of hindsight-rationality condition.

\begin{definition}[Agent Rationality]
The agent satisfies episodic regret exponent $\kappa \in [0,1)$ if there exists $C>0$ and $\zeta>0$ such that for any transfer sequence $(\pi_{\tau,1},\dots,\pi_{\tau,T})$, with probability at least $1 - T^{-\zeta}$,
\[
\sum_{k=1}^T \bigl( V_a^{\pi_a^\star,\pi_{\tau,k}} - V_a^{\pi_{a,k},\pi_{\tau,k}} \bigr) \le C T^\kappa,
\]
where $\pi_a^\star$ is an optimal stationary policy for the agent (under $r_a$).
\end{definition}


\paragraph{Principal’s Objective and Welfare Regret.} Now that we most of the setting is defined, we turn our attention to the objectives that the players have. Formally, we define the global welfare as
\[
W_k = V_a^{\pi_{a,k},\pi_{\tau,k}} + V_p^{\pi_{a,k},\pi_{\tau,k}} \eqsp,
\]
and the social welfare regret is
\[
R_{\mathrm{sw}}(T) = T W^\star - \sum_{k=1}^T W_k \eqsp.
\]
The goal is to design $\pi_{\tau,k}$ such that $R_{\mathrm{sw}}(T) = o(T)$. Again thinking to an insurance company powered with AI algorithms, the welfare would be the utility obtained by both the client (happy to be insured, and ready to pay a fare for that) and the insurer (whose aim is to collect revenues).


\subsection{Principal’s Two-Phase Algorithm}

\paragraph{Phase 1: Transfer Estimation.}  
For each $(s,a)$, the principal seeks the minimal transfer
\[
\tau_s^\star(a) = \max_{a'} \bigl(Q_a(s,a') - Q_a(s,a)\bigr)_+,
\]
where $Q_a$ is the agent’s optimal state–action value function.

Using batched binary search, the principal estimates $\tau_s^\star(a)$ by offering fixed transfers during batches of episodes and observing the fraction of times the agent selects $a$ when in state $s$.

\paragraph{Phase 2: Welfare Optimization.}  
Once the estimates $\hat{\tau}_s(a)$ satisfy
\[
|\hat{\tau}_s(a) - \tau_s^\star(a)| \le T^{-\beta} \eqsp,
\]
the principal can effectively implement any desired action at $s$ by offering $\hat{\tau}_s(a)$. She then runs a no-regret RL algorithm (e.g., UCB-VI) on the shifted MDP with effective rewards 
\[
\tilde r_p(s,a) = r_p(s,a) - \hat{\tau}_s(a),
\]
which preserves welfare. Note that we formulate a simple and theoretical result here, but features can be incorporated while using contextual reinforcement learning algorithms. Pushing things further, we believe that our setting could benefit from Deep RL algorithms.


\subsection{Main Result}

Now that the method and algorithms are exposed, we provide our main theorem, with the proof given in the Appendix.

\begin{theorem}[Social Efficiency in Principal--Agent MDPs]
\label{thm:main_mdp}
Assume the agent satisfies hindsight rationality with exponent $\kappa < 1$ and that the MDP is uniformly ergodic under exploratory policies, ensuring that each state is visited $\Theta(T^\alpha)$ times per batch for some $\alpha>0$. Suppose the principal chooses exponents $\alpha,\beta \in (0,1)$ satisfying
\[
\kappa < \alpha < 1,
\qquad
\frac{\beta}{\alpha} < 1 - \kappa.
\]
Then there exists a two-phase principal’s algorithm such that with high probability:
\begin{align*}
& R_{\mathrm{sw}}(T)
    = O\left( 
      T^\alpha \operatorname{polylog} (T)
      + T^\gamma \operatorname{polylog} (T)\right.
      \\
      & \quad \left. + T^\kappa \operatorname{polylog} (T)
    \right) \eqsp,
\end{align*}
where $\gamma<1$ is the regret exponent of the principal’s RL algorithm in Phase 2. In particular,
\[
R_{\mathrm{sw}}(T) = o(T) \eqsp,
\]
so the principal achieves asymptotically optimal social welfare.
\end{theorem}




\section{Diffusion Models as Welfare Maximizing Economic Mechanisms}

Finally, we now establish a formal connection between diffusion models and classical economic aggregation principles, which completes this work at the intersection between AI and economic theory.  We show that the denoiser learned by a diffusion model is the unique 
solution to a social planner problem under quadratic welfare, and that 
this same object arises as an equilibrium aggregator in a micro-founded 
principal--agents model. 
Thus diffusion models may be interpreted as economic mechanisms for 
aggregating noisy information about latent states.


\subsection{Diffusion Preliminaries}

Let $x_0 \in \mathbb{R}^d$ be drawn from an unknown distribution 
$p_{\mathrm{data}}(x_0)$.  
A (variance-preserving) diffusion model defines a forward noising process
\begin{equation}
    x_t = \alpha_t x_0 + \sigma_t \varepsilon, 
    \qquad \varepsilon \sim \mathcal{N}(0,I_d) \eqsp,
\end{equation}
where $t \in [0,1]$, $\alpha_0 = 1$, $\sigma_0 = 0$, 
and $\alpha_1 \approx 0$, $\sigma_1 \approx 1$.
\newline
\newline
Let $\epsilon_\theta(x_t,t)$ be the denoising network trained via the loss
\begin{equation}
    \cL(\theta)
    =
    \E_{t,x_0,\varepsilon}
    \bigl\| \varepsilon 
    - \epsilon_\theta(\alpha_t x_0 + \sigma_t \varepsilon, t)\bigr\|^2 \eqsp.
\end{equation}
It is classical (for instance in denoising score matching) that the unique minimizer is
\begin{equation}
    \epsilon^\star(x_t,t) 
    = \E[\varepsilon \mid x_t] \eqsp.
\end{equation}
Bayes' rule under the linear-Gaussian model yields
\begin{equation}\label{eq:posterior-x0}
    \mathbb{E}[x_0 \mid x_t] 
    = \frac{1}{\alpha_t}
    \bigl( x_t - \sigma_t\,\epsilon^\star(x_t,t)\bigr) \eqsp,
\end{equation}
which has the great advantage of offering a close form expression.


\subsection{A Social Planner Problem}

Consider a planner who observes only $x_t$ and chooses a reconstruction 
$\hat{x}(x_t) \in \mathbb{R}^d$.  
Define welfare as negative squared error:
\begin{equation}
    W(\hat{x}) 
    \coloneqq 
    -\mathbb{E}\bigl[\|x_0 - \hat{x}(x_t)\|^2\bigr] \eqsp.
\end{equation}
The planner solves
\begin{equation}\label{eq:planner-problem}
    \max_{\hat{x}} W(\hat{x})
    \; \Longleftrightarrow \;
    \min_{\hat{x}} 
    \E\|x_0 - \hat{x}(x_t)\|^2 \eqsp.
\end{equation}

\begin{proposition}[Bayesian Denoiser Maximizes Welfare]
\label{prop:planner-ce}
The unique solution to the planner’s problem 
\eqref{eq:planner-problem} is
\begin{equation}
    \hat{x}^\star(x_t) = \mathbb{E}[x_0 \mid x_t].
\end{equation}
\end{proposition}

\begin{proof}
Fix any measurable $\hat{x}(x_t)$.  
Write
\begin{align*}
& x_0 - \hat{x}(x_t)
=
\bigl(x_0 - \mathbb{E}[x_0 \mid x_t]\bigr)
\\
& \quad + \bigl(\mathbb{E}[x_0 \mid x_t] - \hat{x}(x_t)\bigr).
\end{align*}
Taking squared norms, expanding, and taking expectations:
\begin{align*}
&\E\|x_0 - \hat{x}(x_t)\|^2
= \E\|x_0 - \E[x_0 \mid x_t]\|^2 \\
& \quad + \E\|\E[x_0 \mid x_t] - \hat{x}(x_t)\|^2 \eqsp,
\end{align*}
since the cross-term is zero by the definition of conditional expectation.
The expression is minimized iff 
$\hat{x}(x_t) = \mathbb{E}[x_0 \mid x_t]$ almost surely.
\end{proof}

Using \eqref{eq:posterior-x0}, we obtain the' folowing corolalry.

\begin{corollary}[Diffusion Training = Welfare Maximization]
The minimizer $\epsilon^\star$ of the diffusion loss $\mathcal{L}$ implements 
the welfare-maximizing decision rule $\hat{x}^\star(x_t)$ through the linear 
relationship
\[
    \hat{x}^\star(x_t)
    =
    \frac{1}{\alpha_t}
    \bigl(x_t - \sigma_t \epsilon^\star(x_t,t)\bigr).
\]
Thus solving the diffusion training problem is equivalent to solving 
\eqref{eq:planner-problem}.
\end{corollary}


\subsection{A Micro-Founded Economic Interpretation}

We model a continuum of agents indexed by $i\in[0,1]$.
There is a hidden state $x_0 \in \mathbb{R}^d$.
Each agent observes a signal
\begin{equation}
    y_i = x_0 + \eta_i,
    \qquad \eta_i \sim \mathcal{N}(0,\sigma^2 I),
\end{equation}
independently across $i$. A principal chooses a public decision $d\in\R^d$.
Each agent has utility
\[
u_i(d,x_0) = -\|d - x_0\|^2 \eqsp,
\]
and social welfare is
\[
W(d,x_0) = \int_0^1 u_i(d,x_0)\,di = -\|d - x_0\|^2 \eqsp.
\]
Consider a direct mechanism where each agent reports $m_i$ and the 
principal uses an outcome rule
\[
    d = g(m_{[0,1]})
    = \phi\!\left(\int_0^1 m_i\, di\right).
\]

\begin{lemma}[Truth-Telling in Large Economies]
\label{lemma:truth}
If the principal uses a linear rule 
$d = A \int_0^1 m_i\, di$, 
then with a continuum of agents, truth-telling 
$m_i = y_i$ is a Bayesian Nash equilibrium.
\end{lemma}

\begin{proof}
With a continuum of agents, any individual agent has negligible influence 
on $\int_0^1 m_i\,di$.  
Thus each agent treats the outcome as fixed. The report does not change $d$, so the agent is indifferent among reporting 
strategies; truthful reporting is therefore an equilibrium (whenever one assumes a favorable tie-breaking, which is a very common assumption in such settings).
\end{proof}

Now assume that the principal observes a noisy macro signal generated as 
\begin{equation}
    x_t = \alpha_t x_0 + \sigma_t \varepsilon \eqsp,
\end{equation}
as in the diffusion forward process.  
The principal's welfare problem is again 
$\max_d -\E\|x_0 - d\|^2$.

Combining Proposition \ref{prop:planner-ce} with 
Lemma \ref{lemma:truth} finally leads to the following result.

\begin{proposition}[Diffusion Drift as Welfare-Maximizing Equilibrium]
Under quadratic utilities, symmetric priors, and truth-telling 
(Lemma~\ref{lemma:truth}), the mechanism that maximizes expected welfare 
given noisy macro signal $x_t$ implements
\[
d^\star(x_t) = \mathbb{E}[x_0 \mid x_t],
\]
which corresponds exactly to the diffusion model's optimal denoiser via
\[
d^\star(x_t) 
= \frac{1}{\alpha_t}\bigl(x_t - \sigma_t \epsilon^\star(x_t,t)\bigr).
\]
Thus the reverse-diffusion update direction is the welfare-maximizing
aggregator of noisy private information in a large economy.
\end{proposition}

This section establishes an equivalence between diffusion models and 
economic aggregation: the score or denoiser learned by a diffusion model is characterized by
both (i) a \emph{social planner optimum} under quadratic welfare and 
(ii) an \emph{equilibrium mechanism} in a large-agent economy.  
This provides a principled economic interpretation of diffusion models 
as devices for aggregating dispersed, noisy information about latent 
states, and connects them naturally to incentive design in multi-agent 
learning systems. Although our results linking diffusion models and large-scale principal--agent economics are preliminary, we hope that they offer insights for future research at this intersection

\section{Simulations}

To conclude, we illustrate the role of incentives in a simple principal–agent Markov decision process with a stateful externality. The environment is a finite-horizon line-world in which an agent chooses among three actions: a fast action that advances the agent quickly but generates significant pollution, a slow action with moderate emissions, and a detour action that is costly to the agent but reduces accumulated pollution. Pollution is an explicit state variable that evolves over time and negatively affects the principal’s reward both per step and at the terminal state. The agent, by contrast, values only reaching the goal quickly and does not directly internalize pollution costs. Social welfare is defined as the sum of agent and principal rewards, with monetary transfers canceling out. This setting is a classic illustration of misaligned utilities and distinct roles between a principal and an agent. We show want to show that the incentives allow the players to align their utilities in a favorable way in order to recover global welfare.
\newline
\newline
We compare two settings. In the first, no transfers are offered and the agent learns via Q-learning using only its own rewards. In the second, the principal offers a simple, state-independent subsidy for taking the detour action. This subsidy is calibrated to offset the agent’s private cost of pollution abatement but does not depend on the state or the history of play. In both cases, the agent follows an $\epsilon$-greedy tabular Q-learning algorithm, and performance is evaluated over long-run averages of social welfare and terminal pollution levels.
\newline
\newline
The results, shown in Figure 1 and Figure 2, demonstrate a clear qualitative difference between the two regimes. Without transfers, the agent overwhelmingly favors the fast action, leading to persistent accumulation of pollution and low social welfare. Introducing a simple subsidy substantially alters the agent’s learned behavior: the agent increasingly selects the detour action early in episodes, reducing pollution accumulation over time. As a consequence, average social welfare increases significantly, while the average end-of-episode pollution level decreases. Importantly, these improvements arise despite the subsidy being simply designed.
\newline
\newline
This experiment illustrates a core message of the paper: in multi-agent learning environments with stateful externalities, selfish reinforcement learning can converge to systematically inefficient outcomes, and incentive schemes—even very simple ones—are necessary to align individual learning behavior with social welfare. While the subsidy mechanism used here is intentionally minimal, the observed gains motivate the more structured incentive-compatible mechanisms studied theoretically in the preceding sections.

\begin{figure}[t]
  \centering
  \includegraphics[width=0.48\textwidth]{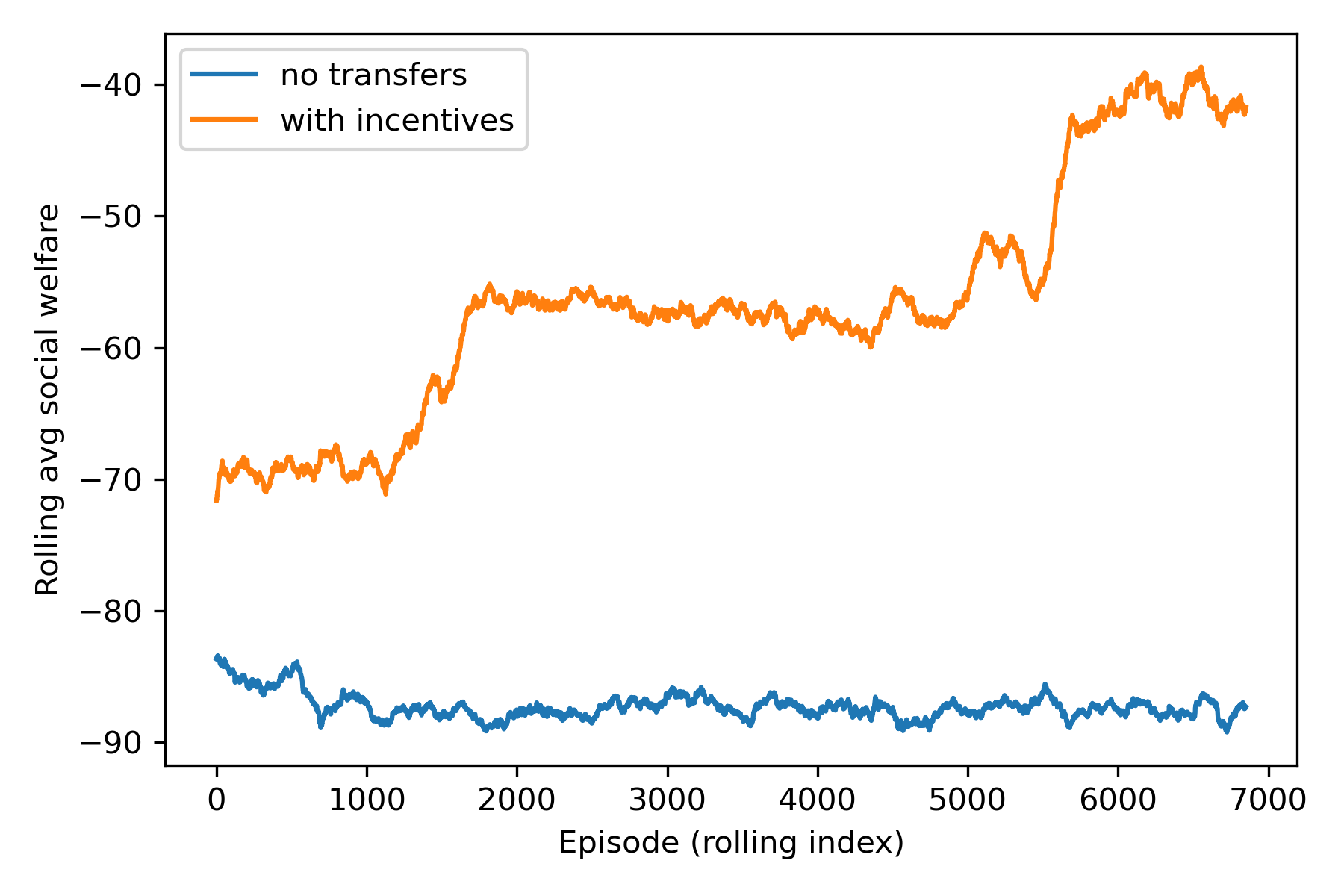}
  \hfill
  \includegraphics[width=0.48\textwidth]{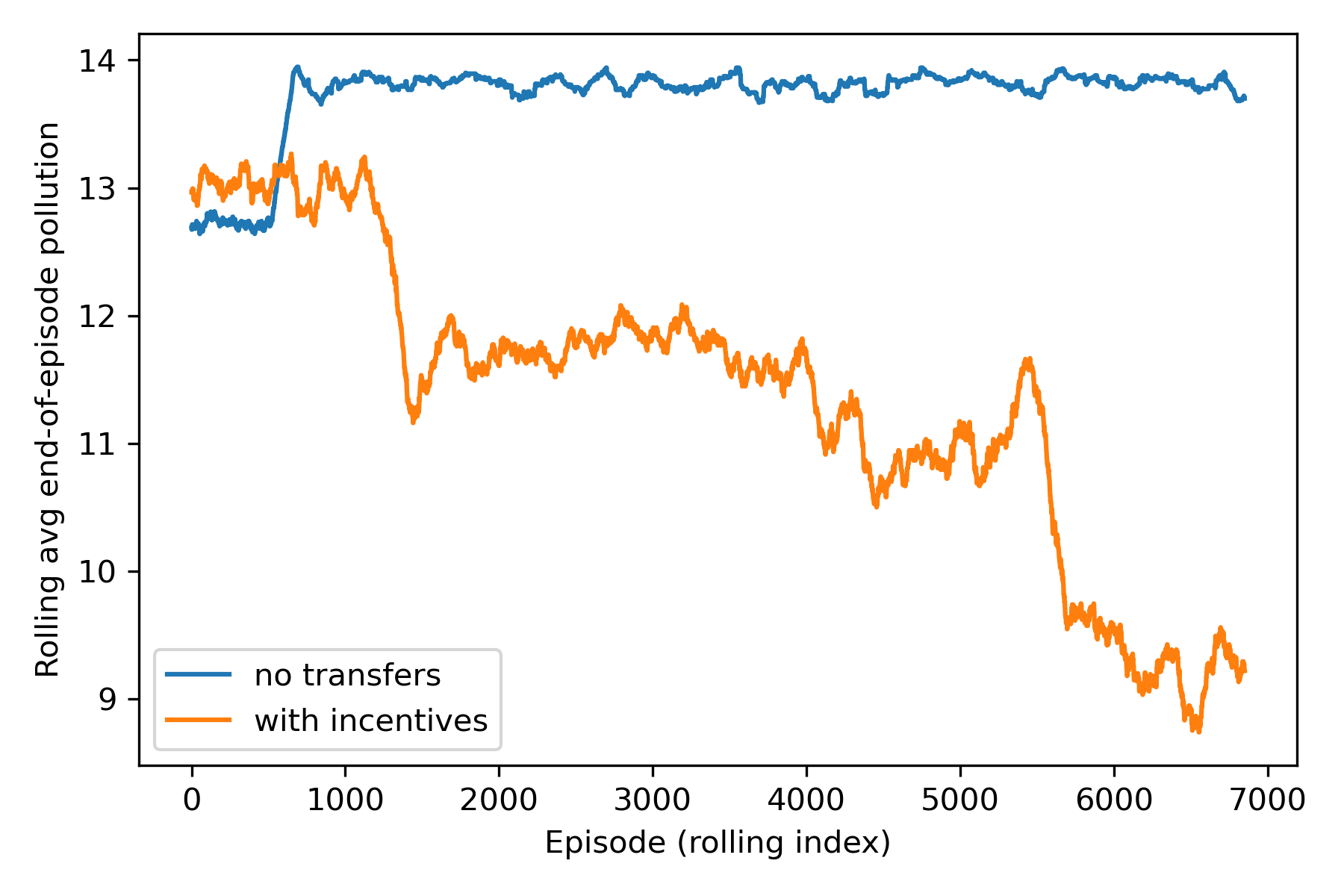}
  \caption{Effect of incentives in a principal–agent MDP with a stateful externality.
  Above: rolling average social welfare. Below: rolling average terminal pollution. Introducing a simple subsidy significantly improves welfare by inducing pollution
  abatement.}
  \label{fig:incentives}
\end{figure}


\section{Conclusion}

We argue that the next generation of AI systems should be studied—and ultimately engineered—as economic mechanisms. When AI mediates prediction, contracting, pricing, and enforcement inside markets and insurance infrastructures, the classical separation between “learning from data” and “designing incentives” collapses. Data become endogenous, behavior responds to policies, and optimization unfolds within a coupled system of interacting learners. Our first contribution is a formal principal–agent framework in a Markov decision process where both players learn over time and where agent actions jointly influence rewards and state transitions. Within this model, we show that transfers—while neutral to welfare ex post—are powerful instruments for welfare alignment ex ante: they reshape the agent’s learning problem so that private incentives internalize externalities. The resulting two-phase mechanism provides a clean conceptual template. In Phase~1, the principal identifies the minimal transfers needed to implement desired actions; in Phase~2, the principal leverages these estimates to effectively steer the long-run dynamics of the system. Under mild regularity conditions, this approach achieves sublinear social-welfare regret. Our second contribution is a conceptual and mathematical bridge between economic aggregation and modern generative modeling.
\newline
\newline
Finally, our simulations illustrate the central message in a transparent environment. Overall, the paper contributes to an emerging view of AI deployment: designing safe and welfare-aligned systems in strategic environments requires co-designing learning dynamics and economic mechanisms.

\newpage
\onecolumn

\bibliographystyle{plainnat}   
\bibliography{sample} 


\section*{Appendix}

\subsection{Proof of Theorem~\ref{thm:main_mdp}}

\paragraph{Step 1: Action Identifiability via Batches.}
Fix $(s,a)$. In a batch of length $L = T^\alpha$, assume the process visits state $s$ at least $\Omega(T^\alpha)$ times (uniform ergodicity). If the offered transfer $\tau$ satisfies $\tau > \tau_s^\star(a) + T^{-\beta}$, then under the agent’s hindsight rationality condition, choosing any $a' \ne a$ in state $s$ incurs regret at least $\Omega(T^\alpha)$, contradicting the regret bound unless misplays occur on at most $O(T^\kappa)$ of the visits. Thus the agent plays $a$ with frequency $1 - O(T^{\kappa-\alpha})$.

Conversely, if $\tau < \tau_s^\star(a) - T^{-\beta}$, then $a$ is suboptimal by at least $T^{-\beta}$, and the agent will choose $a$ at most $O(T^{\kappa-\alpha})$ times. Since $\alpha > \kappa$, these regimes are statistically distinguishable.

\paragraph{Step 2: Batched Binary Search.}
Repeating this test over $O(\log T)$ batches and shrinking the interval for $\tau_s^\star(a)$ by half each time yields an estimate $\hat\tau_s(a)$ with error at most $T^{-\beta}$, provided $\beta/\alpha < 1-\kappa$ to ensure misclassification probability is $o(1)$.

A union bound across all $s,a$ shows that all estimates satisfy
\[
0 \le \hat\tau_s(a) - \tau_s^\star(a) \le 2T^{-\beta} \eqsp,
\]
simultaneously with high probability.

\paragraph{Step 3: Implementability of Desired Actions.}
For any $a'$, we have
\[
Q_a(s,a') \le Q_a(s,a) + \tau_s^\star(a) \le Q_a(s,a) + \hat\tau_s(a) \eqsp,
\]
so $a$ is optimal for the agent whenever the principal offers $\hat\tau_s(a)$. By the regret bound, the agent deviates from $a$ only $o(T)$ times in total during Phase~2.

\paragraph{Step 4: Principal’s RL and Welfare Regret.}
In Phase~2, the principal effectively controls the MDP and faces regret $O(T^\gamma \mathrm{polylog} T)$. Phase~1 contributes at most $O(T^\alpha \mathrm{polylog} T)$ regret, and deviations by the agent contribute $O(T^\kappa \mathrm{polylog} T)$. Thus
\[
R_{\mathrm{sw}}(T)
    = O \left( 
      T^\alpha \operatorname{polylog} T
      + T^\gamma \operatorname{polylog} T
      + T^\kappa \operatorname{polylog} T
    \right) \eqsp,
\]
which is $o(T)$ since $\alpha,\gamma,\kappa<1$. This establishes the theorem. \hfill$\Box$

\end{document}